\RequirePackage{fix-cm}
\documentclass[smallextended]{svjour3}       
\smartqed  
\usepackage{graphicx}

\usepackage{times}
\usepackage{soul}
\usepackage{url}
\usepackage[hidelinks]{hyperref}
\usepackage[utf8]{inputenc}
\usepackage[small]{caption}
\usepackage{graphicx}
\usepackage{amsmath}
\usepackage{booktabs}
\usepackage{algorithmic}
\usepackage[switch]{lineno}

\usepackage[table]{xcolor}
\usepackage{booktabs}
\usepackage{multirow}

\usepackage{enumitem}
\setlist[itemize]{leftmargin=*}
\setlist[enumerate]{leftmargin=*}

\newcommand{\bi}{\begin{itemize}}
\newcommand{\ei}{\end{itemize}}
\newcommand{\be}{\begin{enumerate}}
\newcommand{\ee}{\end{enumerate}}

\usepackage[linesnumbered,ruled,vlined]{algorithm2e}
\usepackage{amsfonts}

\usepackage[skins]{tcolorbox}

\begin{document}

\title{FairReweighing: Density Estimation-Based Reweighing Framework for Improving Separation in Fair Regression
}

\titlerunning{FairReweighing for Fair Regression}        

\author{Xiaoyin Xi         \and
        Zhe Yu 
}


\institute{Xiaoyin Xi \at
              Rochester Institute of Technology \\
              Rochester, New York\\
              \email{xx4455@rit.edu}           
           \and
           Zhe Yu \at
              Rochester Institute of Technology \\
              Rochester, New York\\
              \email{zxyvse@rit.edu}
}

\date{Received: date / Accepted: date}

\maketitle

\begin{abstract}
There has been a prevalence of applying AI software in both high-stakes public-sector and industrial contexts. However, the lack of transparency has raised concerns about whether these data-informed AI software decisions secure fairness against people of all racial, gender, or age groups. Despite extensive research on emerging fairness-aware AI software, up to now most efforts to solve this issue have been dedicated to binary classification tasks. Fairness in regression is relatively underexplored. In this work, we adopted a mutual information-based metric to assess separation violations. The metric is also extended so that it can be directly applied to both classification and regression problems with both binary and continuous sensitive attributes. Inspired by the Reweighing algorithm in fair classification, we proposed a FairReweighing pre-processing algorithm based on density estimation to ensure that the learned model satisfies the separation criterion. Theoretically, we show that the proposed FairReweighing algorithm can guarantee separation in the training data under a data independence assumption. Empirically, on both synthetic and real-world data, we show that FairReweighing outperforms existing state-of-the-art regression fairness solutions in terms of improving separation while maintaining high accuracy.
\end{abstract}

\section{Introduction}
Over the last decades, we have witnessed an explosion in the magnitude and variety of applications upon which statistical machine learning (ML) software is exerted and practiced. The influence of such ML software has drastically reshaped every aspect of our daily life, ranging from ads \cite{perlich2014machine} and movie recommendation engines \cite{biancalana2011context} to life-changing decisions that could lead to severe consequences. For example, physicians and medical professionals can determine if a patient should be discharged based on the predicted probability that they possess the necessity of continuing hospitalization \cite{mortazavi2016analysis}. Similarly, the allocation and personnel of police could depend on future predictions of criminal activity in different communities \cite{redmond2002data}. With its underlying power to affect risk assessment in critical judgment, we inevitably need to be cautious of the potential for such a data-driven algorithm to introduce or, in the worst case, amplify existing discriminatory biases, thus becoming unfair. 

With such a prevalence of potentially biased ML software being developed, it is the responsibility of all AI software developers to make their algorithms accountable by reducing the unwanted biases from ML predictions. 
Given that the definition and criteria for fair decisions vary by context \cite{abu2020contextual}, it is not the responsibility of AI developers to determine whether the training data labels are fair; this responsibility lies with domain experts. However, assuming that the training data labels are accurate and fair, AI developers should ensure that the machine learning software does not perform differently across various sensitive demographic groups. That is, the true positive rate and the false positive rate of the predictions in each demographic group should be the same (equalized odds~\cite{hardt2016equality}) for a fairly designed machine learning software \cite{chakraborty2020fairway,chakraborty2021bias,peng2022fairmask} in binary classification settings. A more general definition of equalized odds is \textit{separation}, which requires the predictions to be conditionally independent of the sensitive attributes given the ground truth dependent variable. This is why, among the various notions of fairness proposed for different scenarios \cite{verma2018fairness,mehrabi2021survey}, this work specifically targets this separation criterion \cite{hardt2016equality}. 

An ML model can inherit the bias from its training data labels, e.g., empirical findings demonstrate that the ML model can inherit the semantics biases humans exhibit in nature languages \cite{caliskan2017semantics}. ML model can also become unfair when it favors one demographic group over another by generating more accurate or positive predictions on data from that group, e.g., it has been found that, in 2020, the face recognition model from large companies including Amazon, Microsoft, IBM, etc. predict in significantly lower accuracy ($20-30\%$) for darker female faces than for lighter male faces \cite{face}. 
Another example is the COMPAS analysis \cite{propublica}, where machine learning software shows similar prediction accuracy across black and white groups in predicting whether a defendant will re-offend within two years. However, the software exhibits a significant disparity in error rates: it has about a 20\% higher false positive rate and a 20\% lower false negative rate for black defendants compared to white defendants. This means that more black defendants are incorrectly predicted to be at higher risk of re-offending when they will not, and more white defendants are incorrectly predicted to be at lower risk of re-offending when they will.

While fairness in binary classification settings has been extensively researched in fair machine learning communities, researchers need to pay more attention to its counterpart in regression problems where a continuous quantity is estimated \cite{berk2017convex}. Most of the optimization on classifier-based systems is established within the context where the decisions are simply binary, for instance, college admission/rejection \cite{kuleto2021exploring}, credit card approval/decline \cite{khandani2010consumer}, or whether a convicted individual would re-offend \cite{propublica}. In contrast to these binary classification problems, there are many influential regression problems such as how much to lend someone on a credit card or home loan, or how much to charge for an insurance premium \cite{steinberg2020fairness}. Most of the existing algorithmic fairness metrics and mitigation solutions cannot be applied directly to these regression problems either because the target and/or predicted values are continuous, or that same value may not occur even twice in the training data \cite{berk2017convex}. Thus, it is essential to develop applicable fairness metrics and mitigation algorithms in regression settings.

Existing fairness metrics in fair regression literature are mostly adaptations of well-established mathematical formulations that were proposed for classification problems, such as Generalized Demographic Parity (GDP) \cite{jiang2022generalized} and statistical parity \cite{agarwal2019fair}. Although demographic parity, where the probability of a certain prediction is the same across groups, guarantees that the predictions of a machine learning model are independent of the membership of a sensitive group, it also promotes laziness-- the oversimplification or lack of rigor in ensuring fair representation across demographic groups by either ignoring individual and contextual factors, or neglecting equity and inclusion. In addition, even if a classifier or regressor is perfectly accurate, it would still reflect some degree of bias in terms of demographic parity as long as the actual class distribution varies among sensitive groups. Instead of demographic parity, we focus on the separation criterion in this work. Firstly, it allows for a perfect predictor, where a model is capable of making flawless predictions for all inputs, when class distributions are different \cite{hardt2016equality}. Secondly, it also penalizes the laziness of demographic parity as it incentivizes the reduction of errors equally across all groups. Steinberg et al.~\cite{steinberg2020fairness} proposed two metrics measuring the violation of separation in regression problems based on probability density estimations. However, those two metrics are still limited since they presuppose the existence of discrete sensitive attributes. They do not apply to datasets with continuous sensitive attributes such as age or income. 

In this paper, we draw inspiration from a mutual information-based separation metric from Steinberg et al.~\cite{steinberg2020fairness} to further extend the metric to scenarios involving continuous sensitive attributes by approximations through a probabilistic regressor. This metric is directly applicable for evaluating the violation of separation in both regression and classification problems with either discrete or continuous sensitive attributes. To better ensure separation (measured by the proposed metric) in fair regression problems, we adapted the commonly applied \emph{Reweighing} algorithm \cite{kamiran12} from classification settings to preprocessed the training data with density estimation (to solve the problem that the same value may not occur even twice in the training data) in regression problems. Thus, this new algorithm \emph{FairReweighing} can support both classification and regression problems with discrete or continuous sensitive attributes. Theoretically, we show that the proposed \emph{FairReweighing} algorithm can guarantee separation in the training data under the assumption of (conditional) data independence. Empirically, by competing against state-of-the-art algorithms in fair regression, we highlight the effectiveness of our proposed algorithm and metric through comprehensive experiments on both synthetic and real-world data sets. In general, the following research questions are explored.

\begin{itemize}
 
  \item \textbf{RQ1: Can the continuous mutual information estimation metric correctly evaluate the violation of separation with continuous sensitive attributes?}
  \item \textbf{RQ2: What is the performance of \emph{FairReweighing} compared to state-of-the-art bias mitigation techniques in regression problems?}
  \item \textbf{RQ3: Is \emph{FairReweighing} performing the same as \emph{Reweighing} in classification problems?}
  \item \textbf{RQ4: What is the performance of FairReweighing when optimizing for multiple continuous sensitive attributes?}
\end{itemize}

Our major contributions are as follows:
\begin{itemize}
  \item We extended the estimation of (conditional) mutual information to scenarios involving continuous sensitive attributes that are universally applicable to both classification and regression models.
  \item We proposed a pre-processing technique \emph{FairReweighing} based on density estimation to satisfy the separation criterion for regression problems.
  \item Our approach demonstrates better performance against state-of-the-art mechanisms on both synthetic and real-world regression data sets.
  \item We also showed that the proposed \emph{FairReweighing} algorithm can be reduced to the well-known algorithm \emph{Reweighing} in classification problems. Therefore \emph{FairReweighing} is a generalized version of \emph{Reweighing} and it can be applied to both classification and regression problems.
  \item The experimental data and the simulation results that support the findings of this study are available at \url{https://github.com/hil-se/FairReweighing}.
\end{itemize}
The rest of this paper is structured as follows. Section~\ref{sect:Related Work} provides the background and related work of this paper. Section~\ref{sec:methodology} introduces the proposed solutions including the new metric evaluating separation with continuous sensitive attributes in Section~\ref{sec:Continuous MI}, and the proposed pre-processing algorithm \emph{FairReweighing} and how it ensures separation in Section~\ref{sect:fairreweighing}. To test the proposed solutions, Section~\ref{sect:Experiments} presents the empirical experiment setups on five datasets and shows the results, and answers the research questions. Followed by a discussion of threats to validity in Section~\ref{threats} and a conclusion in Section~\ref{sect:Conclusion}.

\section{Related Work}
\label{sect:Related Work}

\subsection{Fair Classification}

Most of the existing studies in binary classification settings fall into two categories: suggesting novel fairness definitions to evaluate machine learning models and designing advanced procedures to minimize discrimination without much sacrifice on performance. 
 
\subsubsection{Algorithmic fairness notions in classification problems.}
Most of the existing studies in fair machine learning fall into two categories: suggesting novel fairness definitions to evaluate machine learning models, and designing advanced procedures to minimize discrimination without much sacrifice on performance. \emph{Unawareness} \cite{grgic2016case} demands any decision-making process to exclude the sensitive attribute in the training process. Yet, it suffers from the inference of sensitive characteristics through unprotected traits that serve as proxies \cite{corbett2018measure}. Individual fairness considers the treatments between pairs of individuals rather than comparing at the group level. It was proposed under the principle that "similar individuals should be treated similarly" \cite{dwork2012fairness}. However, because the notion is established on assumptions of the relationship between features and labels, it is hard to find appropriate metrics to compare two individuals \cite{chouldechova2020snapshot}. On the other hand, group fairness is a collection of statistical measurements centering on whether some pre-defined metrics of accuracy are equal across different groups divided by the sensitive attributes. Amongst group fairness notions, \emph{Demographic parity} \cite{dwork2012fairness} requires that the positive rate for the target variable is the same across different demographic groups--- the sensitive attributes $A$ being independent of the prediction $\hat{Y}$. One problem with demographic parity is that it does not allow perfect predictors when the actual class distribution varies among sensitive groups. To always allow perfect predictors, \emph{Separation} is defined as the condition where the sensitive attributes $A$ are conditionally independent of the prediction $\hat{Y}$, given the target value $Y$, and we write:
\begin{equation}\label{separation}
\hat{Y} \perp A \mid Y
\quad \text{Equivalently,} \quad
P(\hat{Y} \mid A, Y) = P(\hat{Y} \mid Y)
\end{equation}
For a binary classifier, separation is equivalent to \emph{equalized odds} \cite{hardt2016equality} which imposes the following two constraints:
\begin{gather}
P(\hat{Y} = 1 \mid Y=1, A=a) = P(\hat{Y} = 1 \mid Y=1, A=b) \\
P(\hat{Y} = 1 \mid Y=0, A=a) = P(\hat{Y} = 1 \mid Y=0, A=b)
\end{gather}
Recall that $P(\hat{Y} = 1 \mid Y=1)$ is referred to as the true positive rate, and $P(\hat{Y} = 1 \mid Y=0)$ as the false positive rate of the classifier. The rest of the paper will focus on the separation criterion since it is a commonly applied group fairness notion and it always allows perfect predictors.

\subsubsection{Algorithms targeting separation in classification}
To achieve separation, researchers have developed a variety of bias mitigation algorithms and frameworks in three different ways: data pre-processing, optimization during software training, or post-processing results of the algorithm. Feldman et al.~\cite{feldman15} proposed a method to test and eliminate disparate impact while preserving important information in the data. Zemel et al.~\cite{zemel13} achieve both group and individual fairness by introducing an intermediate representation of the data and obscuring information on sensitive attributes. \emph{FairBalance} proposed by Yu et al.~\cite{yu2024fairbalance} balance training data distribution across every sensitive attribute to support group fairness. Kamiran et al.~\cite{kamiran12} proposed \emph{Reweighing} by carefully adjusting the influence of the tuples in the training set to be discrimination-free without altering ground-truth labels. In our work, we extended \emph{Reweighing} to work on regression models with proper weight selection and assigning.

An alternative method involves addressing bias during the training phase. This can be achieved by incorporating constraints into the optimization objective of the algorithm. Zhang et al.\cite{zhang2018mitigating} proposed a universal and robust training method based on adversarial learning which optimizes the predictor's capability to forecast $Y$ while concurrently minimizing the adversary's capacity to predict $Z$. Regularization and constraint optimization methods frequently integrate fairness considerations into the classifier loss function, working on the confusion matrix during the model training process.

The ultimate approach aims to rectify the outcomes of a classifier to attain fairness. In this method, a classifier provides a score for each individual, and the objective is to make binary predictions based on these scores. Calibration involves ensuring that the ratio of positive predictions aligns with the ratio of positive examples \cite{dawid1982well}. In the fairness context, this alignment must extend to all subgroups, whether they are considered sensitive or not, within the dataset \cite{chen2018my}. Thresholding is another post-processing strategy motivated by the observation that biased decision-makers often make discriminatory decisions near decision boundaries \cite{kamiran2012decision}. This approach is grounded in the understanding that humans commonly apply threshold rules when making decisions \cite{kleinberg2018human}.

\subsection{Fair Regression}
Defining fairness metrics in regression settings has always been challenging. Unlike classification, no consensus on its formulation has yet emerged. The significant difference from the prior classification setup is that the target variable is now allowed to be continuous rather than binary or categorical.

\subsubsection{Metrics of separation in regression problems.}

Berk et al. \cite{berk2017convex} introduce a flexible family of pairwise fairness regularizers for regression problems based on individual notions of fairness in classification settings where similar individuals should be treated similarly. Agarwal et al. \cite{agarwal2019fair} propose general frameworks for achieving fairness in regression under two fairness criteria, statistical parity, which requires that the prediction is statistically independent of the sensitive attribute, and bounded group loss, which demands that the prediction error for any sensitive group does not exceed a pre-defined threshold.

Jiang et al.\cite{jiang2022generalized} present Generalized Demographic Parity (GDP) as a fairness metric for both continuous and discrete attributes that preserve tractable computation and justify its unification with the well-established demographic parity. However, it also inherits its disadvantage of ruling out a perfect predictor and promoting laziness when there is a causal relationship between the sensitive attribute and the output variable. Another recent work by Narasimhan et al.\cite{narasimhan2020pairwise} proposes a collection of group-dependent pairwise fairness metrics for ranking and regression models. They translate classification-exclusive statistics like true positive rate and false positive rate into the likelihood of correctly ranking pairs and develop corresponding metrics.

Most existing separation criterion ($\hat{Y}\perp A | Y$) are equivalent to equalized odds in classification settings and were generalized to regression settings, however almost none of them applies to situations concerning continuous sensitive features. Steinberg et al.~\cite{steinberg2020fairness} present the following two methods for efficiently and numerically approximating the separation criteria within the broader context of regression. 

\textbf{Ratio of separation} is a metric evaluation the violation of separation in the form of density ratios,
\begin{equation}
\mathbb{E}_{\hat{Y}}[r_{sep}] = \frac{P(\hat{Y} \mid A=1, Y)}{P(\hat{Y} \mid A=0, Y)},
\end{equation}
where a ratio of 1 would represent the most fair. The probability densities are approximated by density ratios from the outputs of two probabilistic classifiers:
\begin{gather}
\label{eq:approx_1}
P(A = a \mid \hat{Y} = \hat{y}) \approx \rho(a \mid \hat{y}) \\ 
\label{eq:approx_2}
P(A = a \mid Y = y, \hat{Y} = \hat{y}) \approx \rho(a \mid y, \hat{y})
\end{gather}
Here, $\rho(a \mid \hat{y})$ and $\rho(a \mid y, \hat{y})$ represent the predictions of the probability where $A=a$, which is generated by introducing and training two machine learning classifiers (one with $\hat{y}$ as the input and the other one with $\hat{y}$ and $y$ as inputs). Now, we can use such density ratio estimates to approximate $r_{sep}$ by using Bayes’ rule, \eqref{eq:approx_1} and \eqref{eq:approx_2}
\begin{equation}\label{rsep}
\hat{r}_{sep} = \frac{1}{n} \sum_{i=1}^{n} \frac{\rho(1 \mid y_i, \hat{y}_i)}{1-\rho(1 \mid y_i, \hat{y}_i)} \cdot \frac{1-\rho(1 \mid y_i)}{\rho(1 \mid y_i)}
\end{equation}
where $i$ represents a data point, and the formula suggests that it does not apply to sensitive attributes that are continuous. The separation approximations gauge the additional predictive power that the joint distribution of $Y$ and $\hat{Y}$ provides in determining $A$, compared to solely considering the marginal distributions of $Y$. These techniques are designed to be versatile, working independently of the specific regression algorithms employed. 

\textbf{(Conditional) mutual information approximation.} 
Steinberg et al.~\cite{steinberg2020fairness} also presented another alternative method for assessing fairness criteria, which mitigates certain constraints associated with the direct density ratio approach mentioned above, and involves computing the mutual information \cite{cover1991elements} between variables. In the realms of probability theory and information theory, the mutual information between two random variables serves as a gauge of the interdependence shared between these variables. To evaluate the independence criterion outlined in \eqref{separation}, the conditional mutual information between variables $\hat{Y}$ and $A$ can be computed,
\begin{equation}
    I[\hat{Y};A \mid Y] = \int_{Y} \int_{\hat{Y}} \sum_{a \in \mathcal{A}} P(y,\hat{y},a) \log \frac{P(\hat{y},a \mid y)}{P(\hat{y}\mid y)P(a\mid y) } dyd\hat{y}
\end{equation}

Here, it is evident that under conditions of independence and perfect fairness, the joint probability $P(\hat{Y}, A \mid Y)$ equals the product of the marginal probabilities $P(\hat{Y} \mid Y)$ and $P(A \mid Y)$, resulting in $I[\hat{Y};A \mid Y] = 0$. Conversely, if there is a lack of independence, the mutual information (MI) will be positive. This metric seamlessly accommodates binary and categorical $A$. Once again, employing empirical estimation and leveraging the probabilistic classifiers, the conditional mutual information for separation is,
\begin{equation}\label{mi}
\hat{I}_{sep} = \frac{1}{n} \sum_{i=1}^n \log \frac{\rho(a_i\mid y_i, \hat{y}_i)}{\rho(a_i \mid y_i)}
\end{equation}
We can observe that these approximate measures rely on the probabilistic classifiers. In contrast to the direct density ratio measures, they exhibit a more natural operation with categorical sensitive attributes through the use of multiclass probabilistic classification, but not continuous ones.
In summary, the two separation metrics in \eqref{rsep} and \eqref{mi} from Steinberg et al.~\cite{steinberg2020fairness} do not apply to problems with continuous sensitive attributes. Therefore in Section~\ref{sect:GMI} we extend the mutual information metric in \eqref{mi} to scenarios with continuous sensitive attributes by proposing the continuous mutual information estimation metric.

\subsubsection{Bias mitigation in regression}

Categorized as in-processing approaches, Berk et al.~\cite{berk2017convex}, Agarwal et al.~\cite{agarwal2019fair} and Jiang et al.~\cite{jiang2022generalized} all introduce a regularization (or penalty) term in the training process and aim to find the model which minimizes the associated regularized loss function, where Narasimhan et al.~\cite{narasimhan2020pairwise} directly enforce a desired pairwise fairness criterion using constrained optimization. Calders et al.~\cite{calders2013controlling} also proposed a propensity score-based stratification approach for balancing the dataset to address discrimination-aware regression. They define two measures for quantifying attribute effects and then develop constrained linear regression models for controlling the effect of such attributes on the predictions. It is essentially an in-processing mechanism that solves an optimization problem between predicting accuracy and discrimination.

In this paper, we propose (in Section~\ref{sect:fairreweighing}) a pre-processing technique to satisfy the separation criterion. By reweighing the training data before training any regression model, we show that separation can be satisfied. This technique is more efficient and has significantly low overhead competing against other mechanisms. It is also considered more generalized as there is no limitation on which learning algorithms should be applied. This approach is free of trade-off or \emph{Price of Fairness} parameters (weights) and can achieve a competitively better balance between sensitive attributes and prediction accuracy.

\section{Methodology}\label{sec:methodology}
\subsection{Continuous Mutual Information Estimation}
\label{sec:Continuous MI}

\label{sect:GMI}
Compared with the direct density ratio measures of separation, the conditional mutual information approximation proposed by Steinberg et al.~\cite{steinberg2020fairness} can be applied to categorical sensitive attributes instead of binary ones using multiclass probabilistic classification. Until now, much of the research on fairness has primarily concentrated on safeguarding categorical variables. In this study, we loosen this assumption and present a continuous version of mutual information able to handle continuous sensitive attributes, $A\in \mathbb{R}$.

We adopt the same mathematical formulation of mutual information in \eqref{mi} using empirical estimation,
\begin{equation}\label{mi_con}
\hat{C}_{sep} = \frac{1}{n} \sum_{i=1}^n \log \frac{\rho(a_i\mid y_i, \hat{y}_i)}{\rho(a_i \mid y_i)}.
\end{equation}
Different from Steinberg et al.~\cite{steinberg2020fairness}, when operating on continuous sensitive attributes, we will not be able to use classifiers to estimate the conditional densities per instance as in the case of the direct density ratio estimation. The target $A$ is a continuous quantity and there could be an infinite number of instances. To approximate such density, instead of relying on the output of a probabilistic classifier, we introduce and train two linear regression models $f_1(Y, \hat{Y})$ and $f_2(\hat{Y})$ 
 to predict the sensitive attribute $A$. 
As shown in Figure \ref{fig:linear}, $\rho(a_i\mid y_i, \hat{y}_i)$ and ${\rho(a_i \mid y_i)}$ can be estimated as
\begin{equation}
\begin{aligned}
\rho(a_i | y_i, \hat{y}_i) &\hat{=} \frac{1}{\sqrt{2\pi \sigma_1^2}}e^{-\frac{(a_i-f_1(y_i, \hat{y}_i))^2}{2\sigma_1^2}} \\
\rho(a_i | \hat{y}_i) &\hat{=} \frac{1}{\sqrt{2\pi \sigma_2^2}}e^{-\frac{(a_i-f_2(\hat{y}_i))^2}{2\sigma_2^2}}.
\end{aligned}
\end{equation}
where $\sigma_1$ and $\sigma_2$ are estimated by the residual errors of $f_1$ and $f_2$:
\begin{equation*}
\begin{aligned}
\sigma_1 &\hat{=} \sqrt{\frac{1}{n}\sum_{i=1}^n (a_i-f_1(y_i, \hat{y}_i))^2} \\
\sigma_2 &\hat{=} \sqrt{\frac{1}{n}\sum_{i=1}^n (a_i-f_2(\hat{y}_i))^2}.
\end{aligned}
\end{equation*}

\begin{figure}[ht]
  \centering
  \includegraphics[width=\linewidth]{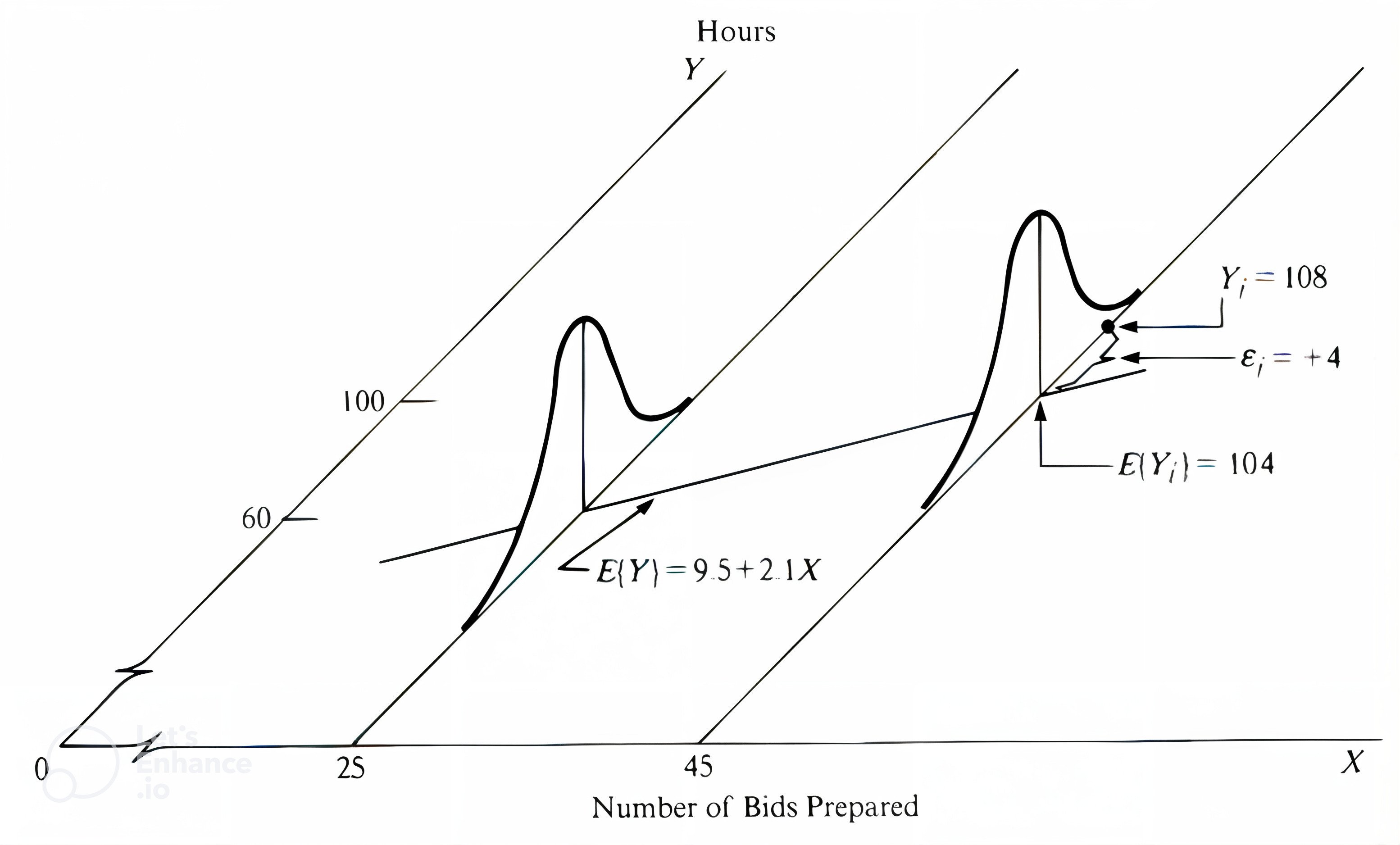}
  \caption{Conditional Distribution of Response in Regression Model}
  \label{fig:linear}
\end{figure}

\subsection{FairReweighing}
\label{sect:fairreweighing}
The Reweighing algorithm proposed by Kamiran et al.~\cite{kamiran12} preprocesses the training data with different sampling weights for each demographic group with different dependent variables:
\begin{equation}
\label{reweighing}
W(a, y) = \frac{ P(a) P(y)}{P(a, y)}.
\end{equation}
Reweighing requires very little computation overhead and does not need any trade-off parameter when it ensures separation for the learned model on the training data--- we will show this later in \textbf{Proposition 1}. However, while the probabilities of $P(a)$, $P(y)$, and $P(a,y)$ can be easily estimated by frequency counts in the binary classification setting, they are difficult to estimate in regression settings when both $a$ and $y$ can be continuous--- the same value may only occur once in the training data. One solution is to specify an interval threshold and segment the training data into a set of \(N\) brackets. Then we can approximate it into a multiclass classification problem with \(y_1...y_N\) target groups. However, such a binning or bucketing mechanism throws down another challenge on where the lines between bins should be drawn, and it might introduce hyperparameters that need to be tuned manually. To accurately capture the above probabilities, we adopt two different nonparametric probability density estimation techniques in FairReweighing.

\textbf{FairReweighing (Neighbor):} The first approach, radius neighbors \cite{goldberger2004neighbourhood}, is an enhancement to the well-known k-nearest neighbors algorithm that evaluates density by taking into account all the samples within a defined radius of a new instance instead of just its k closest neighbors. By specifying the radius range and the metric to use for distance computation, we locate all neighbors of a given sample and use that as an estimation of its density:
\begin{equation}
\rho_N(x) = N(x_i \mid dist(x,x_i)<r) \quad \forall i \in \{1...N\}
\end{equation}

\textbf{FairReweighing (Kernel):} Another method we implement is kernel density estimation (KDE) \cite{terrell1992variable}, which is a widely used non-parametric method for estimating the probability density function of a continuous random variable. It works by placing a kernel (a smooth, symmetric function, typically a Gaussian) at each data point in the dataset. The KDE is then computed by summing the contributions of all the kernels placed at each data point. The result is a smooth curve that represents the estimated probability density function. The KDE at a point 
$x$ is given by:
\begin{equation}\label{kde}
\hat{f}(x) = \frac{1}{nh} \sum_{i=1}^n K (\frac{x - x_i}{h})
\end{equation}
Where $n$ is the number of data points. $h$ is the bandwidth (smoothing parameter). And $K$ is the kernel function.

Whichever density estimation is administered through the pipeline, \emph{FairReweighing} is a generalized algorithm that can be applied to both classification and regression problems. Furthermore, it is capable of constraining multiple sensitive attributes in specific applications. For example, it can separation concerning gender and race simultaneously. Because both radius neighbors and kernel density estimation can be performed in any number of dimensions, our approach would automatically calculate the extent to which each feature should be reweighted to meet the separation criteria. 
After the imbalance between the observed and expected probability density has been equalized as in \eqref{reweighing}, we use the calculated weights on our training data to learn for a discrimination-free regressor or classifier. The procedure outlining \emph{FairReweighing} is detailed in Algorithm \ref{alg:FairReweighing}.
\begin{algorithm}
\DontPrintSemicolon 
\SetKwInOut{Input}{Input}
\SetKwInOut{Output}{Output}
\SetKwInOut{Parameter}{Parameter}
\SetKwRepeat{Do}{do}{while}
\Input{$(X\in \mathbb{R}^{n},A\in \mathbb{R},Y\in \mathbb{R})$, Given }
\Output{Regressor or classifier learned on reweighted $(X,A,Y)$ }
\For{$a_i \in A$}{
    \For{$y_i \in Y$}{
        $W(a_i, y_i) = \frac{\rho (a_i) \times \rho (y_i)}{\rho (a_i, y_i)}$
    }
}
\For{x $\in$ X}{
    $W(x) = W(x(A), x(Y))$
}
Train a regressor or classifier on $(X,A,Y)$ with sample weight $W(X)$\;
\Return Regressor or Classifier $M$
\caption{\emph{FairReweighing}}\label{alg:FairReweighing}
\end{algorithm}

\subsubsection{How FairReweighing ensures separation}\label{sec:proof}

\textbf{Problem Statement.} Consider training a model on a dataset $(X\in \mathbb{R}^n,\,A\in \mathbb{R},\,Y\in \mathbb{R})$, the model makes predictions of $\hat{Y}\in \mathbb{R}$ based on its inputs $(x,a)$. The goal is to have the model's prediction $\hat{Y}$ satisfy the separation criterion in \eqref{separation}.

\begin{proposition}
Under the conditional independence assumption $X\perp A|Y$, a model trained with FairReweighing satisfies separation on its training data. 
\end{proposition}
\begin{proof}
To check whether separation is satisfied for the model on its training data, we have
\begin{equation}
\label{psy}
\begin{aligned}
P(\hat{Y}=\hat{y} | y) &= \iint P(\hat{Y}=\hat{y} | x, a)P(x,a|y)dxda\\
P(\hat{Y}=\hat{y} | a, y) &= \int P(\hat{Y}=\hat{y} | x, a)P(x|a, y)dx.
\end{aligned}
\end{equation}
This is because the model's prediction $\hat{Y}$ is completely decided by the input $(x, a)$ and thus $\hat{Y}\perp Y | (X, A)$. Under conditional independence assumption of the data $X\perp A | Y$, we have
\begin{equation}
\label{psy2}
\begin{aligned}
P(\hat{Y}=\hat{y} | y) &= \iint P(\hat{Y}=\hat{y} | x, a)P(x|y)P(a|y) dxda\\
P(\hat{Y}=\hat{y} | a, y) &= \int P(\hat{Y}=\hat{y} | x, a)P(x|y) dx.
\end{aligned}
\end{equation}
With FairReweighing in \eqref{reweighing}, the model learns from the weighted training data:
\begin{equation}
\label{pred2}
\begin{aligned}
P(\hat{Y} = \hat{y} | x, a) &=  \frac{P(Y=\hat{y}, x, a)W(a,\hat{y})}{\int P(Y=y, x, a)W(a,y)dy}.
\end{aligned}
\end{equation}
With \eqref{reweighing} and the independence assumption of $X\perp A | Y$, we have
\begin{equation}
\label{pw}
\begin{aligned}
P(y, x, a)W(a,y) &= \frac{ P(a) P(y)P(y, x, a)}{P(a, y)}\\
&= P(a)P(y)P(x|a,y)\\
&= P(a)P(y)P(x|y)= P(a)P(x,y). 
\end{aligned}
\end{equation}
Apply \eqref{pw} to \eqref{pred2} we have
\begin{equation}
\label{pred3}
\begin{aligned}
P(\hat{Y} = \hat{y} | x, a) &= \frac{P(a)P(x,\hat{y})}{\int P(a)P(x,y) dy} = \frac{P(x,\hat{y})}{\int P(x,y) dy} \\
&= \frac{P(x,\hat{y})}{P(x)}  = P(\hat{y}|x).
\end{aligned}
\end{equation}
Apply \eqref{pred3} to \eqref{psy2} we have
\begin{equation}
\label{psy3}
\begin{aligned}
P(\hat{Y}=\hat{y} | a, y) &= \int P(\hat{y} | x)P(x|y)dx\\
P(\hat{Y}=\hat{y} | y) &= \iint P(\hat{y} | x)P(x|y)P(a|y) dxda\\ 
&= \int P(\hat{y} | x)P(x|y)dx \cdot \int P(a|y) da \\
&=\int P(\hat{y} | x)P(x|y) dx = P(\hat{Y}=\hat{y} | a, y).
\end{aligned}
\end{equation}
Therefore $\hat{Y}\perp A|Y$ and the model satisfies separation. 
\end{proof}

\section{Experiment}
\label{sect:Experiments}

In this section, we answer our research questions through experiments on both synthetic and real-world data in regression settings. For radius neighbors, we find the best radius at 0.5 through cross-validation and calculate the Euclidean distance between the target sample and other data points. When implementing KDE, we first standardize features by subtracting the mean and scaling to unit variance, then proceed with Gaussian as kernel function and bandwidth of 0.2 based on the cross-validation result of $\hat{C}_{sep}$ on the training data. Linear regression is applied as the regressor given its estimation errors (MSE) on these small tabular datasets. Each data set is divided randomly into training and test sets with a ratio of 50:50, and we report the average of any given metric over 20 independent iterations. The algorithms were implemented in Python and all experiments were run on a linux machine with a 5.8 GHz Intel i9 processor, GeForce RTX 4090 GPU and 64GB 6000MHz DDR5 memory. We compare \textbf{FairReweighing} against (a) Berk et al.~\cite{berk2017convex} which is designed to achieve convex individual fairness; (b) Agarwal et al.~\cite{agarwal2019fair} which is designed to achieve statistical (or demographic) parity (DP) in fair classification and regression; (c) Narasimhan et al.~\cite{narasimhan2020pairwise} which is designed to achieve a pairwise fairness metric. 
Apart from comparing against these state-of-the-art bias mitigation techniques, we also include a none-treatment method as a baseline in all experiments. An overview of the datasets can be seen in Table \ref{tab:data}.\footnote{Code is available at: \url{https://github.com/hil-se/FairReweighing}.}
\begin{table}[ht]
\small
  \setlength\tabcolsep{2pt}
    \setlength\extrarowheight{3pt}
  \caption{Data overview}
  \label{tab:data}
  \centering
  \begin{tabular}{l|c|c|c|c|c}
  \hline \multirow{2}{*}{Type} & \multirow{2}{*}{Data} & \multirow{2}{*}{Total size} & \multirow{2}{*}{\# features} & \multicolumn{2}{c}{Sensitive Attributes}  \\ \cline{5-6}
  &&&& Binary & Continuous\\
  \hline
  
  \multirow{3}{*}{Regression} & \emph{Synthetic} & 5000 & 3 & Gender & Age  \\ 
  
   & \emph{Law} & 22,407 & 10 & Race & \\ 

   & \emph{Crimes} & 1994 & 119 & Race & Race\%\\ 

  \hline


  \hline
  \end{tabular}
\end{table}

We also experimented on more complex and larger real-world dataset with VGG-16 model but the results are not that valuable since $\hat{C}_{sep}$ is already minimal without any treatment. Detailed results are included in the appendix.

To answer \textbf{RQ1} and \textbf{RQ2}, we study how \emph{FairReweighing} (with a linear regression model trained on the pre-processed data) performs against the state-of-the-art bias mitigation algorithms for regression. Every treatment is evaluated by the $\hat{r}_{sep}$ metric from \eqref{rsep}, $\hat{I}_{sep}$ from \eqref{mi} and $\hat{C}_{sep}$ from \eqref{mi_con} --- the closer $\hat{r}_{sep}$ is to $1.0$ and $\hat{I}_{sep}$, $\hat{C}_{sep}$ to $0$ , the better the model satisfies separation. Note that $\hat{r}_{sep}$ and $\hat{I}_{sep}$ are only applicable to binary sensitive attribute settings while $\hat{C}_{sep}$ can handle continuous sensitive attributes. We conducted comprehensive experiments on one synthetic data set and two real-world data sets.

\vspace{0.2cm} 
\begin{tabular}{rl}
$\text{gender}_i$ &$\sim  Bernoulli(0.7)$ \\
$\text{age}_i$ &$\sim  Normal(40,15)$ \\

$\text{height}_i \mid \text{gender}_i = 1$ & $\sim Normal(1.75, 0.15)$ \\
$\text{height}_i \mid \text{gender}_i = 0$ &$\sim Normal(1.65, 0.1)$ \\
$\text{power}_i \mid \text{gender}_i = 1$ &$\sim Normal(0.6, 0.15)$ \\
$\text{power}_i \mid \text{gender}_i = 0$ &$\sim Normal(0.5, 0.1)$ \\
$\text{jump}$ &$\sim \frac{(\text{height}+\text{power}) \times 40}{age}$
\end{tabular}

\vspace{0.2cm} 

\subsubsection{Synthetic Data}

Our synthetic data set is constructed as above. The following are the critical assumptions for generating this data: vertical jump height is primarily determined by the height of the tester and his/her power level, and women tend to be shorter and have less power level than men on average, same with age. Our task is to predict a subject's vertical jump height based on gender, age, and height (power as a hidden feature) and evaluate proposed fairness metrics on sensitive attributes. The reason we construct such a naive data set is to test the accuracy of our density estimation methods by comparing the estimations with the ground truth distributions.

\begin{figure}[ht]
  \centering

  \includegraphics[width=\linewidth]{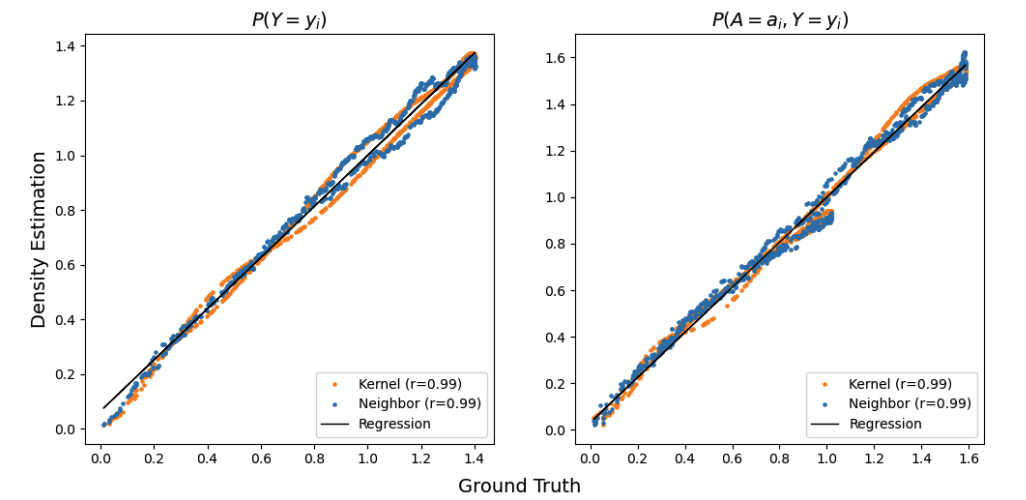}
  \caption{Comparison between density estimations and the ground truth distribution.}
  \label{fig:Synthetic}
\end{figure}

Figure \ref{fig:Synthetic} demonstrates the relationship between the two density estimations for $P(Y=y_i)$ and $P(A=a_i, Y=y_i)$ with their corresponding ground truth distributions when treating gender as the sensitive attribute. Both estimations have a strong positive correlation with a Pearson correlation coefficient close to 1 ($r=0.99$). 
Table \ref{tab:synthetic} shows the accuracy and fairness metric when implementing \emph{FairReweighing} with either ground truth distribution or density estimation approximation. The result suggests that \emph{FairReweighing} with both density estimations successively improve separation while maintaining good prediction accuracy.

\begin{table*}[ht]
\small
  \setlength\tabcolsep{4pt}
    \setlength\extrarowheight{1pt}

  \caption{Results on the synthetic regression problem.}
  \label{tab:synthetic}
  \centering
  \begin{tabular}{l|c|c|c|c|c|c|c}
  \hline {Bias Mitigation Treatment} & MSE & $R^2$ & BGL & $\hat{r}_{sep}$(G) &$\hat{I}_{sep}$(G) & $\hat{C}_{sep}$(G) & $\hat{C}_{sep}$(A)\\ \hline
  
  {None} & 0.019 & 0.594 & 0.021 & 1.580 & 0.157 & 0.089 & 0.124\\ 

  {\emph{FairReweighing} (Ground Truth)} & 0.020 & 0.566 & 0.022 & 1.027 & 0.008  & 0.002 & 0.003\\
  
  {\emph{FairReweighing} (Neighbor)} & 0.020 & 0.577 & 0.024 & 1.075 & 0.017 & 0.008 & 0.005\\
  
  {\emph{FairReweighing} (Kernel)} & 0.019 & 0.555 & 0.024 & 1.064 & 0.016 & 0.008 &0.008\\  
  \cline{2-8} \hline
  \end{tabular}
\end{table*}

\subsubsection{Real World Data}
\begin{itemize}
    \item The \emph{Communities and Crimes} \cite{redmond2002data} data contains 1,994 instances of socio-economic, law enforcement, and crime data about communities in the United States. There are nearly 140 features, and the task is to predict each community's per capita crime rate.
    Practically, when the prediction of crime rate from the regression model trained on this data is being used to decide the necessary police strength in each community, ethical issues may arise when the separation criterion is not satisfied for the model's prediction and whether the community is white dominant or not. This means the predictions are incorrectly affected by the sensitive attribute for communities with the same crime rates. We will show in our experiment how separation will be evaluated and tackled by FairReweighing.

    \item The \emph{Law School} \cite{wightman1998lsac} data records information of 22,407 law school students, and we are trying to predict a student's normalized first-year GPA based on his/her LSAT score, family income, full-time status, race, gender and the law school cluster the student belongs to. In our experiment, we treat race as a discrete sensitive attribute, divided into black and white students. 
    Practically, when the prediction of first-year GPA from the regression model trained on this data is being used to decide whether a student should be admitted, ethical issues may arise when the separation criterion is not satisfied for the model's prediction and gender or race. This means the predictions are incorrectly affected by these sensitive attributes for students with the same first-year GPAs. We will show in our experiment how separation will be evaluated and tackled by FairReweighing.
    
\end{itemize}

\begin{table*}[ht]
\small
  \setlength\tabcolsep{6pt}
    \setlength\extrarowheight{1pt}
  \caption{Results on real world regression problems.}
  \label{tab:binary}
  \centering
  \begin{tabular}{l|c|c|c|c|c|c|c}
  \hline Data & Group & {Bias Mitigation Treatment} & MSE & $R^2$ & $\hat{r}_{sep}$ & $\hat{I}_{sep}$ & $\hat{C}_{sep}$\\ \hline  

    \multirow{6}{*}{\emph{Law}} & \multirow{6}{*}{Race} & {None} &  0.073 &0.526 &1.240 &0.056&0.028\\   
  & & {Berk et al.} & 0.075 &0.518&1.096&0.023&0.015\\ 
  & & {Agarwal et al.} &  0.075&0.522&1.099&0.024&0.016 \\ 
  & & {Narasimhan et al.} &  0.073&0.533&1.113&0.036&0.020\\   
  & & {\emph{FairReweighing} (Neighbor)} &  0.074&0.543&1.032 &0.014&0.009 \\   
  & & {\emph{FairReweighing} (Kernel)} &  0.079&0.516&1.036 &0.017&0.009 \\ \hline
  \multirow{13}{*}{\emph{Crimes}} & \multirow{6}{*}{Race} & {None} & 0.020 & 0.626 & 1.579 & 0.191& 0.099  \\   
  & & {Berk et al.} & 0.024 & 0.553 & 1.130 & 0.043 & 0.020 \\ 
  & & {Agarwal et al.} &  0.024 & 0.548 & 1.134 & 0.046 & 0.014  \\ 
  & & {Narasimhan et al.} &  0.029&0.474&1.166 & 0.065 & 0.044 \\   
  & & {\emph{FairReweighing} (Neighbor)} &  0.024&0.562&1.099 & 0.022& 0.007 \\   
  & & {\emph{FairReweighing} (Kernel)} &  0.023&0.567&1.105 & 0.023& 0.007 \\ \cline{2-8}
 &\multicolumn{4}{c|}{}& \multicolumn{3}{c}{$\hat{C}_{sep}$}\\ \cline{2-8}

   & \multirow{6}{*}{Race\%} & {None} & 0.020 & 0.632 & \multicolumn{3}{c}{0.205} \\   
  & & {Berk et al.} & 0.024 & 0.553 & \multicolumn{3}{c}{0.055} \\ 
  & & {Agarwal et al.} &  0.024 & 0.548 & \multicolumn{3}{c}{0.048} \\ 
  & & {Narasimhan et al.} &  0.029&0.474&\multicolumn{3}{c}{0.077} \\   
  & & {\emph{FairReweighing} (Neighbor)} &  0.024&0.532&\multicolumn{3}{c}{0.011} \\   
  & & {\emph{FairReweighing} (Kernel)} &  0.025&0.538&\multicolumn{3}{c}{0.013} \\ \hline  
  \end{tabular}
\end{table*}

\subsubsection{Results} 
Table \ref{tab:synthetic} and Table \ref{tab:binary} show the average accuracy and fairness metric on data sets in the regression setting with both synthetic and real-world datasets.

\paragraph{\textbf{RQ1}}
By comparing the trend of $\hat{C}_{sep}$ with $\hat{r}_{sep}$ and $\hat{I}_{sep}$, we can see that it follows the same trajectory of decreasing when applying different bias mitigation techniques. This demonstrates that the metric that we proposed, $\hat{C}_{sep}$, is consistent with $\hat{r}_{sep}$ and $\hat{I}_{sep}$ when the sensitive attribute is discrete, and it also accurately correctly evaluate the
violation of separation with continuous sensitive attributes.

\paragraph{\textbf{RQ2}}
Predictions generated by a linear regression model trained after \emph{FairReweighing} are significantly less biased in terms of $\hat{r}_{sep}$, $\hat{I}_{sep}$ and $\hat{C}_{sep}$ against other treatments for both density estimation method. Across all three data sets, we see an average of 71\% reduction in all three measures compared to all other treatments. In terms of accuracy, our proposed treatments retain high prediction accuracy in all cases, with MSE and $R^2$ scores comparable to the non-treatment baseline. Although we did not fully minimize $\hat{r}_{sep}$ to 1, or $\hat{I}_{sep}$/$\hat{C}_{sep}$ to 0 (no discrimination at all) in some cases, the results still demonstrate that our approach performs better than the state-of-the-art bias mitigation technique in terms of all metrics and can always consistently balances accuracy and fairness. Hence, \emph{FairReweighing} generates the most fair predictions concerning the separation criterion when compared to state-of-the-art bias mitigation techniques in regression problems.

It is also worth mentioning that none of the baselines compared in Table~\ref{tab:binary} were specifically designed to optimize for separation. While Berk et al.~\cite{berk2017convex} has been used to improve $\hat{r}_{sep}$ in Steinberg et al.~\cite{steinberg2020fairness}, it was originally designed to optimize for individual fairness. This also explains why FairReweighing achieves better separation when compared to other baselines--- it is the first treatment specifically designed to optimize for separation in regression problems.

\subsubsection{Results with deep neural networks}

To validate the generalizability of the proposed metric and \emph{FairReweighing} algorithm in complex ML models, we also performed experiments on a facial beauty prediction dataset SCUT-FBP5500 with 5,500 face images and their corresponding attractiveness ratings. We used a VGG-16 model for the prediction. Table~\ref{tab:SCUT_avg} and ~\ref{tab:SCUT_R2} are the results when training and testing with average ratings of 60 raters, and ratings from the P8 rater. The results were not that informative because separation is naturally satisfied even without \emph{FairReweighing}-- the model trained on P8's ratings has the largest violation of separation before applying \emph{FairReweighing}. However, it still demonstrates that our mechanism applies to larger real-world datasets under more complex models.

\begin{table*}[ht]
    \setlength\extrarowheight{3pt}
  \caption{Average accuracy and fairness metrics for each sensitive attributes on SCUT-FBP5500 training on average ratings from 60 raters}
  \label{tab:SCUT_avg}
  \begin{center}
  \begin{tabular}{l|c|c|c|c|c|c|c}
  \hline \multirow{2}{*}{Treatment} & \multirow{2}{*}{MSE} & \multirow{2}{*}{Pearson $r$} & \multirow{2}{*}{Spearman's $r_s$} & \multicolumn{2}{c|}{Sex} & \multicolumn{2}{c}{Race} \\ \cline{5-8}
  
  & & & & $\hat{I}_{sep}$ & $\hat{C}_{sep}$ & $\hat{I}_{sep}$ & $\hat{C}_{sep}$ \\ \hline
  
  {None} & 0.092 & 0.804 & 0.811 & 0.026 & 0.018 & 0.006 & 0.004 \\ 
  
  {\emph{FR} (Neighbor)} & 0.101 & 0.796 & 0.769 & 0.013 & 0.011 & 0.002 &  0.003 \\
  
  {\emph{FR} (Kernel)} & 0.097 &0.801 &0.796 &0.023 &0.012 &0.003 &0.001 \\  \hline
  \end{tabular}
  \end{center}
\end{table*}

  
  
  
  

\begin{table*}[ht]
    \setlength\extrarowheight{3pt}
  \caption{Average accuracy and fairness metrics for each sensitive attributes on SCUT-FBP5500 training on ratings from P8.}
  \label{tab:SCUT_R2}
  \begin{center}
  \begin{tabular}{l|c|c|c|c|c|c|c}
  \hline \multirow{2}{*}{Treatment} & \multirow{2}{*}{MSE} & \multirow{2}{*}{Pearson $r$} & \multirow{2}{*}{Spearman's $r_s$} & \multicolumn{2}{c|}{Sex} & \multicolumn{2}{c}{Race} \\ \cline{5-8}
  
  & & & & $\hat{I}_{sep}$ & $\hat{C}_{sep}$ & $\hat{I}_{sep}$ & $\hat{C}_{sep}$ \\ \hline
  
  {None} & 0.188 & 0.563 & 0.536 & 0.046 &  0.022 & 0.003  & 0.009 \\ 
  
  {\emph{FR} (Neighbor)} & 0.173 & 0.594 & 0.545 & 0.038 & 0.021 & 0.012 &  0.004\\
  
  {\emph{FR} (Kernel)} & 0.169 & 0.626 & 0.584 & 0.023 & 0.022 & 0.001 &  0.004\\  \hline
  \end{tabular}
  \end{center}
\end{table*}

\subsection{Classification}
\subsubsection{Data sets}
To answer \textbf{RQ3}, we used two publicly available data sets in fair classification literature that assume either binary or continuous sensitive groups:
\begin{itemize}
    \item The \emph{German Credit} \cite{german} data consists of 1000 instances, with 30\% of them being positively classified into good credit risks. There are nine features, and the task is to predict an individual's creditworthiness based on his/her economic circumstances. Gender or age is considered the sensitive attribute in our experiments.
    \item The \emph{Heart Disease} \cite{german} data consists of 1190 instances with 14 attributes. The major task of this data set is to predict whether there is the presence of heart disease or not (0 for no disease, 1 for disease) based on the given attributes of a patient. In our experiments, gender or age is treated as a sensitive attribute.
\end{itemize}

\begin{table*}[ht]
\small
  \setlength\tabcolsep{4pt}
  \setlength\extrarowheight{2pt}
  \caption{Results on real world classification problems.}
  \label{tab:classification}
  \begin{center}
  \begin{tabular}{l|c|c|c|c|c|c|c|c|c}
  \hline Data & Group & {Models} & Accuracy & F1 &AOD & EOD&  $\hat{r}_{sep}$ &$\hat{I}_{sep}$ & $\hat{C}_{sep}$\\ \hline
  
  \multirow{9}{*}{\emph{German}} & \multirow{4}{*}{Gender} & {None} & 0.626
 & 0.688 & 0.071 & 0.021 & 1.038 & 0.011 & 0.007 \\ 
  
  & & {\emph{Reweighing}} & 0.640 & 0.699 & 0.004 &0.031& 1.013 & 0.004 & 0.003 \\
  
  & & {\emph{FairReweighing}(Neighbor)} & 0.640 & 0.699 & 0.004 &0.031& 1.013 & 0.004 & 0.003  \\
  
   & & {\emph{FairReweighing}(Kernel)} & 0.640 & 0.699 & 0.004 &0.031& 1.013 & 0.004 & 0.003 \\ \cline{2-10}
   
   &\multicolumn{4}{c|}{}& \multicolumn{5}{c}{$\hat{C}_{sep}$}\\ \cline{2-10}
   
 & \multirow{4}{*}{Age} & {None} & 0.623
 
 & 0.685 & \multicolumn{5}{c}{0.012} \\

  & & {\emph{FairReweighing}(Neighbor)} & 0.625 & 0.689 & \multicolumn{5}{c}{0.004}\\
  
   & & {\emph{FairReweighing}(Kernel)} & 0.625 & 0.689 & \multicolumn{5}{c}{0.004}\\ \hline
  
  \multirow{9}{*}{\emph{Heart}} & \multirow{4}{*}{Gender} & {None} & 0.828  & 0.845 & 0.073 & 0.120 &1.071 & 0.018 & 0.010 \\ 
  
  & & {\emph{Reweighing}} & 0.807 & 0.831 & 0.048 & 0.010 & 1.013  & 0.005 & 0.003 \\ 
  
  & & {\emph{FairReweighing}(Neighbor)} & 0.807 & 0.831 & 0.048 & 0.010 & 1.013  & 0.005 & 0.003 \\
  
  & & {\emph{FairReweighing}(Kernel)} & 0.807 & 0.831 & 0.048 & 0.010 & 1.013  & 0.005 & 0.003 \\ \cline{2-10}
  
   &\multicolumn{4}{c|}{}& \multicolumn{5}{c}{$\hat{C}_{sep}$}\\ \cline{2-10}

   & \multirow{4}{*}{Age} & {None} & 0.835 & 0.855 & \multicolumn{5}{c}{0.008} \\

  & & {\emph{FairReweighing}(Neighbor)} & 0.822 & 0.844 & \multicolumn{5}{c}{0.005} \\
  
  & & {\emph{FairReweighing}(Kernel)} & 0.822 & 0.844 & \multicolumn{5}{c}{0.005} \\ \hline
  \end{tabular}
  \end{center}
\end{table*}

\subsubsection{Results (\textbf{RQ3})} 
Table \ref{tab:classification} shows the average results of 20 independent trials in classification. 
\begin{itemize}
\item
Here in this table, we used $\hat{r}_{sep}$, $\hat{I}_{sep}$, $\hat{C}_{sep}$, Average Odds Difference (AOD), and Equal Opportunity Difference (eOD) to measure separation in classification problems. AOD and EOD are metrics measuring the violation of equalized odds. The closer they are to $0$, the better the model satisfies separation. Table~\ref{tab:classification} shows that all metrics are always consistent. This validates that, $\hat{r}_{sep}$, $\hat{I}_{sep}$ and $\hat{C}_{sep}$ are applicable to measure separation for classification.
\item
We can also observe that all four measurements show the same trend--- after applying either \emph{Reweighing} or \emph{FairReweighing}, separation will be better satisfied while maintaining prediction accuracy (except for EOD in German data which is already close to 0). More specifically, every metric is the same for \emph{Reweighing} and \emph{FairReweighing} in classification problems. This validates that, \emph{FairReweighing} is identical as \emph{Reweighing} for classification with binary sensitive attributes. While \emph{Reweighing} can only be applied to classification problems with binary sensitive attributes.
\end{itemize}

\subsection{Multiple Sensitive Attributes}

Finally, we consider the scenario of constraining multiple sensitive attributes simultaneously. Again, we use the \emph{Communities and Crimes}~\cite{redmond2002data} data set that has the percentage of the population for each ethnicity group in a community as continuous sensitive attributes. Instead of focusing solely on African Americans, we include the percentage of White and Asian populations and evaluate the fairness of all three sensitive attributes simultaneously.

\begin{table*}[ht]
  \setlength\extrarowheight{2pt}
  \caption{Average accuracy and fairness metrics for each sensitive attributes on Communities and Crimes data.}
  \label{tab:multiple}
    \begin{center}

  \begin{tabular}{l|c|c|c|c|c}
  \hline \multirow{2}{*}{Models} & \multirow{2}{*}{MSE} & \multirow{2}{*}{$R^2$} & Black & White & Asian \\ \cline{4-6}
  
  & & &  $\hat{C}_{sep}$ & $\hat{C}_{sep}$ & $\hat{C}_{sep}$ \\ \hline
  
  {None} & 0.020 & 0.635  & 0.201  & 0.274  & 0.002    \\
  
  {\emph{FairReweighing}} (Neighbor) & 0.025 & 0.529  & 0.051 & 0.048 & 0.002  \\
  
  {\emph{FairReweighing}} (Kernel) & 0.027 & 0.493 & 0.022  & 0.022  & 0.002 \\ \hline
  \end{tabular}
  \end{center}
\end{table*}

\subsubsection{Results (\textbf{RQ4})} The average results for 20 trials of prediction accuracy and respective $\hat{C}_{sep}$ for each sensitive attribute are displayed in Table \ref{tab:multiple}. We notice that \emph{FairReweighing} achieves much lower fairness violations in terms of the percentage of the black and white population compared to the baseline while failing the bias mitigation task for the percentage of Asians. The underperforming is because no discrimination was detected against the Asian population even before pre-processing. To balance fairness between all three attributes, slight bias might be introduced to compensate for the black and white population. As for model performance, both of our proposed treatments preserve relatively high accuracy in predictions, with only approximately 15\% sacrifices. 
It is also important to note that our fairness metrics correctly reflect the orientation of imbalance using the positive or negative number, with the black population being deprived and white being favored in this case. Overall, our algorithm successfully mitigates bias for multiple continuous sensitive attributes.

\section{Threats to validity}
\label{threats}

\noindent\textbf{Sampling Bias} - Conclusions may change if other datasets and machine learning models are used. We have attempted to reduce the sampling bias by experimenting on five different datasets (including one synthetic dataset) and using the same linear regression models across all experiments.

\noindent\textbf{Evaluation Bias} - We focused on the equalized odds fairness notion and evaluated it with $\hat{r}_{sep}$, $\hat{I}_{sep}$ and $\hat{C}_{sep}$. For scenarios where other types of fairness are required, e.g. demographic parity, the proposed algorithm may not apply. 

\noindent\textbf{Estimation Bias} - Both probability density estimation techniques employed in FairReweighing may be impacted by the efficacy of the estimation itself. Conclusions may change if other density estimation techniques are used.

\noindent\textbf{External Validity} - This work focuses on regression problems which are very common in AI software. We are currently working on extending it to machine learning problems in comparative settings.

\section{Conclusion and Future Work}
\label{sect:Conclusion}

Ever since the discovery of underlying discriminatory issues with data-driven methods, the great majority of work in research communities on fairness in machine learning has centered around classic classification problems, where the target variable is either binary or categorical. However, it is only one facet of how machine learning is utilized. In this work, we introduced a pre-processing algorithm \emph{FairReweighing} specifically designed to achieve the separation criterion in regression problems. Furthermore, we make use of mutual information to tackle challenges related to continuous sensitive attributes by estimating the output using a probabilistic regressor. This metric is directly suitable for measuring the extent of separation breach in both regression and classification problems with either discrete or continuous sensitive attributes. Through comprehensive experiments on both synthetic and real-world data sets, \emph{FairReweighing} outperforms existing state-of-the-art algorithms in terms of satisfying separation for regression. 

In our future work, we will 
We also show that \emph{FairReweighing} is reduced to the well-known pre-processing algorithm \emph{Reweighing} when applied to classification problems. 
This work suffers from several limitations:
\begin{enumerate}[label=\textbf{Limitation \arabic*}]
\item The approximations used in either mutual information or direct density ratio measures can be influenced by the effectiveness of the classifier or regressor $\rho(a \mid \cdot)$. Determining whether a classifier's poor performance is due to a fair assessment or a suboptimal model selection can be challenging.
\item \emph{FairReweighing} wasn't able to fully minimize $\hat{r}_{sep}$ to 1, or $\hat{I}_{sep}$/$\hat{C}_{sep}$ to 0 (no discrimination at all) in some cases. This is probably due to the difference in the training and test data. There is still room for improvement.
\item Both probability density estimation techniques employed in \emph{FairReweighing} may be impacted by the efficacy of the estimation itself. It can be challenging to discern whether an inadequate performance is a result of a fair evaluation or suboptimal model selection.
\item Separation only ensures that the model is fair for a given set of ground truth. However, when the ground truth labels are provided by e.g. biased human decisions, the model may inherit the bias from the human decisions even when it satisfies separation with respect to the human decisions.
\end{enumerate}

\section{Declarations}

\subsection{Funding}
Funding in direct support of this work: NSF grant 2245796.

\subsection{Ethical approval}

This study did not require ethical approval as it was based solely on the analysis of publicly available data that did not involve any interaction with human or animal subjects, and contained no personally identifiable information.

\subsection{Informed consent}

Informed consent was not required for this study as it involved the use of publicly available datasets that do not contain personally identifiable information.

\subsection{Author Contributions}
Both authors contributed equally to the work.

\subsection{Data Availability Statement}
The experimental data and the simulation results that support the findings of this study are available at \url{https://github.com/hil-se/FairReweighing}.

\subsection{Conflict of Interest}

This work was supported by National Science Foundation, but the funders had no role in the design, data collection, analysis, interpretation, or publication of this study. The authors declare no other conflicts of interest.

\subsection{Clinical Trial Number}

Clinical trial number: not applicable.

\bibliographystyle{spmpsci}      
\bibliography{mybib}   

@article{kamiran12,
  title={Data preprocessing techniques for classification without discrimination},
  author={Kamiran, Faisal and Calders, Toon},
  journal={Knowledge and information systems},
  volume={33},
  number={1},
  pages={1--33},
  year={2012},
  publisher={Springer}
}

@inproceedings{feldman15,
  title={Certifying and removing disparate impact},
  author={Feldman, Michael and Friedler, Sorelle A and Moeller, John and Scheidegger, Carlos and Venkatasubramanian, Suresh},
  booktitle={proceedings of the 21th ACM SIGKDD international conference on knowledge discovery and data mining},
  pages={259--268},
  year={2015}
}

@inproceedings{zemel13,
  title={Learning fair representations},
  author={Zemel, Rich and Wu, Yu and Swersky, Kevin and Pitassi, Toni and Dwork, Cynthia},
  booktitle={International conference on machine learning},
  pages={325--333},
  year={2013},
  organization={PMLR}
}

@article{yu2024fairbalance,
  title={FairBalance: How to Achieve Equalized Odds With Data Pre-processing},
  author={Yu, Zhe and Chakraborty, Joymallya and Menzies, Tim},
  journal={IEEE Transactions on Software Engineering},
  year={2024},
  publisher={IEEE}
}

@article{redmond2002data,
  title={A data-driven software tool for enabling cooperative information sharing among police departments},
  author={Redmond, Michael and Baveja, Alok},
  journal={European Journal of Operational Research},
  volume={141},
  number={3},
  pages={660--678},
  year={2002},
  publisher={Elsevier}
}

@article{wightman1998lsac,
  title={LSAC National Longitudinal Bar Passage Study. LSAC Research Report Series.},
  author={Wightman, Linda F},
  year={1998},
  publisher={ERIC}
}

@article{perlich2014machine,
  title={Machine learning for targeted display advertising: Transfer learning in action},
  author={Perlich, Claudia and Dalessandro, Brian and Raeder, Troy and Stitelman, Ori and Provost, Foster},
  journal={Machine learning},
  volume={95},
  number={1},
  pages={103--127},
  year={2014},
  publisher={Springer}
}

@incollection{biancalana2011context,
  title={Context-aware movie recommendation based on signal processing and machine learning},
  author={Biancalana, Claudio and Gasparetti, Fabio and Micarelli, Alessandro and Miola, Alfonso and Sansonetti, Giuseppe},
  booktitle={Proceedings of the 2nd Challenge on Context-Aware Movie Recommendation},
  pages={5--10},
  year={2011}
}

@article{mortazavi2016analysis,
  title={Analysis of machine learning techniques for heart failure readmissions},
  author={Mortazavi, Bobak J and Downing, Nicholas S and Bucholz, Emily M and Dharmarajan, Kumar and Manhapra, Ajay and Li, Shu-Xia and Negahban, Sahand N and Krumholz, Harlan M},
  journal={Circulation: Cardiovascular Quality and Outcomes},
  volume={9},
  number={6},
  pages={629--640},
  year={2016},
  publisher={Am Heart Assoc}
}

@article{caliskan2017semantics,
  title={Semantics derived automatically from language corpora contain human-like biases},
  author={Caliskan, Aylin and Bryson, Joanna J and Narayanan, Arvind},
  journal={Science},
  volume={356},
  number={6334},
  pages={183--186},
  year={2017},
  publisher={American Association for the Advancement of Science}
}

@article{berk2017convex,
  title={A convex framework for fair regression},
  author={Berk, Richard and Heidari, Hoda and Jabbari, Shahin and Joseph, Matthew and Kearns, Michael and Morgenstern, Jamie and Neel, Seth and Roth, Aaron},
  journal={arXiv preprint arXiv:1706.02409},
  year={2017}
}

@inproceedings{grgic2016case,
  title={The case for process fairness in learning: Feature selection for fair decision making},
  author={Grgic-Hlaca, Nina and Zafar, Muhammad Bilal and Gummadi, Krishna P and Weller, Adrian},
  booktitle={NIPS symposium on machine learning and the law},
  volume={1},
  pages={2},
  year={2016}
}

@article{corbett2018measure,
  title={The measure and mismeasure of fairness: A critical review of fair machine learning},
  author={Corbett-Davies, Sam and Goel, Sharad},
  journal={arXiv preprint arXiv:1808.00023},
  year={2018}
}

@inproceedings{dwork2012fairness,
  title={Fairness through awareness},
  author={Dwork, Cynthia and Hardt, Moritz and Pitassi, Toniann and Reingold, Omer and Zemel, Richard},
  booktitle={Proceedings of the 3rd innovations in theoretical computer science conference},
  pages={214--226},
  year={2012}
}

@article{hardt2016equality,
  title={Equality of opportunity in supervised learning},
  author={Hardt, Moritz and Price, Eric and Srebro, Nati},
  journal={Advances in neural information processing systems},
  volume={29},
  year={2016}
}

@article{chouldechova2020snapshot,
  title={A snapshot of the frontiers of fairness in machine learning},
  author={Chouldechova, Alexandra and Roth, Aaron},
  journal={Communications of the ACM},
  volume={63},
  number={5},
  pages={82--89},
  year={2020},
  publisher={ACM New York, NY, USA}
}

@inproceedings{narasimhan2020pairwise,
  title={Pairwise fairness for ranking and regression},
  author={Narasimhan, Harikrishna and Cotter, Andrew and Gupta, Maya and Wang, Serena},
  booktitle={Proceedings of the AAAI Conference on Artificial Intelligence},
  volume={34},
  number={04},
  pages={5248--5255},
  year={2020}
}

@article{goldberger2004neighbourhood,
  title={Neighbourhood components analysis},
  author={Goldberger, Jacob and Hinton, Geoffrey E and Roweis, Sam and Salakhutdinov, Russ R},
  journal={Advances in neural information processing systems},
  volume={17},
  year={2004}
}

@article{terrell1992variable,
  title={Variable kernel density estimation},
  author={Terrell, George R and Scott, David W},
  journal={The Annals of Statistics},
  pages={1236--1265},
  year={1992},
  publisher={JSTOR}
}

@misc{german,
author = "Dua, Dheeru and Graff, Casey",
year = "2017",
title = "{UCI} Machine Learning Repository",
url = "http://archive.ics.uci.edu/ml",
institution = "University of California, Irvine, School of Information and Computer Sciences" }

@inproceedings{agarwal2019fair,
  title={Fair regression: Quantitative definitions and reduction-based algorithms},
  author={Agarwal, Alekh and Dud{\'\i}k, Miroslav and Wu, Zhiwei Steven},
  booktitle={International Conference on Machine Learning},
  pages={120--129},
  year={2019},
  organization={PMLR}
}

@inproceedings{jiang2022generalized,
  title={Generalized demographic parity for group fairness},
  author={Jiang, Zhimeng and Han, Xiaotian and Fan, Chao and Yang, Fan and Mostafavi, Ali and Hu, Xia},
  booktitle={International Conference on Learning Representations},
  year={2022}
}

@article{steinberg2020fairness,
  title={Fairness measures for regression via probabilistic classification},
  author={Steinberg, Daniel and Reid, Alistair and O'Callaghan, Simon},
  journal={arXiv preprint arXiv:2001.06089},
  year={2020}
}

@inproceedings{calders2013controlling,
  title={Controlling attribute effect in linear regression},
  author={Calders, Toon and Karim, Asim and Kamiran, Faisal and Ali, Wasif and Zhang, Xiangliang},
  booktitle={2013 IEEE 13th international conference on data mining},
  pages={71--80},
  year={2013},
  organization={IEEE}
}

@book{cover1991elements,
  title={Elements of Information Theory},
  author={Cover, T.M. and Thomas, J.A.},
  isbn={9780471062592},
  lccn={lc90045484},
  series={Wiley Series in Telecommunications and Signal Processing},
  url={https://books.google.com/books?id=CX9QAAAAMAAJ},
  year={1991},
  publisher={Wiley}
}

@inproceedings{zhang2018mitigating,
  title={Mitigating unwanted biases with adversarial learning},
  author={Zhang, Brian Hu and Lemoine, Blake and Mitchell, Margaret},
  booktitle={Proceedings of the 2018 AAAI/ACM Conference on AI, Ethics, and Society},
  pages={335--340},
  year={2018}
}

@article{dawid1982well,
  title={The well-calibrated Bayesian},
  author={Dawid, A Philip},
  journal={Journal of the American Statistical Association},
  volume={77},
  number={379},
  pages={605--610},
  year={1982},
  publisher={Taylor \& Francis}
}

@article{chen2018my,
  title={Why is my classifier discriminatory?},
  author={Chen, Irene and Johansson, Fredrik D and Sontag, David},
  journal={Advances in neural information processing systems},
  volume={31},
  year={2018}
}

@inproceedings{kamiran2012decision,
  title={Decision theory for discrimination-aware classification},
  author={Kamiran, Faisal and Karim, Asim and Zhang, Xiangliang},
  booktitle={2012 IEEE 12th international conference on data mining},
  pages={924--929},
  year={2012},
  organization={IEEE}
}

@article{kleinberg2018human,
  title={Human decisions and machine predictions},
  author={Kleinberg, Jon and Lakkaraju, Himabindu and Leskovec, Jure and Ludwig, Jens and Mullainathan, Sendhil},
  journal={The quarterly journal of economics},
  volume={133},
  number={1},
  pages={237--293},
  year={2018},
  publisher={Oxford University Press}
}

@MISC{propublica,
  title={Machine Bias: There's software used across the country to predict future criminals. And it's biased against blacks},
  author= {Julia Angwin and Jeff Larson and Surya Mattu and Lauren Kirchner},
  howpublished = {\url{https://www.propublica.org/article/machine-bias-risk-assessments-in-criminal-sentencing}},
  publisher = {https://www.propublica.org/},
  year={2016}
}

@MISC{face,
  title={Racial Discrimination in Face Recognition Technology},
  author= {Najibi, Alex},
  url = {https://sitn.hms.harvard.edu/flash/2020/racial-discrimination-in-face-recognition-technology},
  publisher = {https://sitn.hms.harvard.edu/},
  year={2020}
}

@article{abu2020contextual,
  title={Contextual fairness: A legal and policy analysis of algorithmic fairness},
  author={Abu-Elyounes, Doaa},
  journal={U. Ill. JL Tech. \& Pol'y},
  pages={1},
  year={2020},
  publisher={HeinOnline}
}

@inproceedings{chakraborty2020fairway,
  title={Fairway: a way to build fair ML software},
  author={Chakraborty, Joymallya and Majumder, Suvodeep and Yu, Zhe and Menzies, Tim},
  booktitle={Proceedings of the 28th ACM Joint Meeting on European Software Engineering Conference and Symposium on the Foundations of Software Engineering},
  pages={654--665},
  year={2020}
}

@inproceedings{chakraborty2021bias,
author = {Chakraborty, Joymallya and Majumder, Suvodeep and Menzies, Tim},
title = {Bias in Machine Learning Software: Why? How? What to Do?},
year = {2021},
isbn = {9781450385626},
publisher = {Association for Computing Machinery},
address = {New York, NY, USA},
url = {https://doi.org/10.1145/3468264.3468537},
doi = {10.1145/3468264.3468537},
booktitle = {Proceedings of the 29th ACM Joint Meeting on European Software Engineering Conference and Symposium on the Foundations of Software Engineering},
pages = {429–440},
numpages = {12},
location = {Athens, Greece},
series = {ESEC/FSE 2021}
}

@article{peng2022fairmask,
  title={FairMask: Better Fairness via Model-based Rebalancing of Protected Attributes},
  author={Peng, Kewen and Chakraborty, Joymallya and Menzies, Tim},
  journal={IEEE Transactions on Software Engineering},
  year={2022},
  publisher={IEEE}
}

@article{mehrabi2021survey,
  title={A survey on bias and fairness in machine learning},
  author={Mehrabi, Ninareh and Morstatter, Fred and Saxena, Nripsuta and Lerman, Kristina and Galstyan, Aram},
  journal={ACM Computing Surveys (CSUR)},
  volume={54},
  number={6},
  pages={1--35},
  year={2021},
  publisher={ACM New York, NY, USA}
}

@inproceedings{verma2018fairness,
  title={Fairness definitions explained},
  author={Verma, Sahil and Rubin, Julia},
  booktitle={Proceedings of the international workshop on software fairness},
  pages={1--7},
  year={2018}
}

@article{kuleto2021exploring,
  title={Exploring opportunities and challenges of artificial intelligence and machine learning in higher education institutions},
  author={Kuleto, Valentin and Ili{\'c}, Milena and Dumangiu, Mihail and Rankovi{\'c}, Marko and Martins, Oliva MD and P{\u{a}}un, Dan and Mihoreanu, Larisa},
  journal={Sustainability},
  volume={13},
  number={18},
  pages={10424},
  year={2021},
  publisher={MDPI}
}

@article{khandani2010consumer,
  title={Consumer credit-risk models via machine-learning algorithms},
  author={Khandani, Amir E and Kim, Adlar J and Lo, Andrew W},
  journal={Journal of Banking \& Finance},
  volume={34},
  number={11},
  pages={2767--2787},
  year={2010},
  publisher={Elsevier}
}

%
%

\end{document}